\documentclass[accepted]{uai2026} 

\usepackage[american]{babel}
\usepackage{natbib}
\usepackage{amsmath,amssymb,amsthm}
\usepackage{booktabs}
\usepackage{graphicx}
\usepackage{mathtools}
\usepackage{algorithm}
\usepackage{algpseudocode}

\newcommand{\E}{\mathbb{E}}
\newcommand{\Prob}{\mathbb{P}}

\newcommand{\1}{\mathbf{1}}
\newcommand{\cI}{\mathcal{I}}
\newcommand{\cD}{\mathcal{D}}
\newcommand{\cM}{\mathcal{M}}

\newcommand{\cY}{\mathcal{Y}}

\newcommand{\cH}{\mathcal{H}}
\newcommand{\cF}{\mathcal{F}}

\newcommand{\defeq}{\coloneqq}
\newcommand{\iid}{\overset{\text{i.i.d.}}{\sim}}

\newtheorem{theorem}{Theorem}
\newtheorem{lemma}{Lemma}
\newtheorem{proposition}{Proposition}

\title{Certified Interventional Fidelity:\\
Anytime-Valid, Adaptive Evaluation of Causal Claims in Mechanistic Interpretability}
\author[1]{\href{mailto:amir.asiaeetaheri@vumc.org?Subject=Your UAI 2026 paper}{Amir~Asiaee}{}}
\affil[1]{%
    Department of Biostatistics\\
    Vanderbilt University Medical Center\\
    Nashville, TN 37232, USA
}

\begin{document}
\maketitle

\begin{abstract}
Mechanistic interpretability often evaluates explanations by intervening on a model: swapping hidden states,
patching activations, ablating components, or comparing a compressed model to the original one. These
experiments are usually summarized by a point estimate, even though the evaluation may be monitored while it
runs or adapted toward suspected failures. This makes it hard to tell whether a reported fidelity or patching
effect is a stable causal claim or a consequence of finite sampling and evaluation choices.

We introduce \textbf{Certified Interventional Fidelity (CIF)}, a statistical layer for interventional
interpretability evaluations. CIF first writes the quantity being reported as a causal estimand: an expectation
of a bounded score over a stated input distribution and a stated intervention distribution. It then provides
confidence intervals and anytime-valid confidence sequences for this estimand, including under adaptive
intervention sampling via bounded mixture importance weighting. We instantiate CIF with Hoeffding-style
sequences and variance-adaptive betting sequences, the latter reducing certification cost by $10$--$30\times$
in our experiments. On MNIST abstractions and GPT-2 Small IOI circuits, CIF certifies high-fidelity claims,
shows when apparent method differences are not statistically supported, and makes sensitivity to the
intervention distribution explicit.
\end{abstract}

\section{Introduction}

Deep neural networks can match or exceed human performance, but understanding \emph{how} they compute remains
challenging.
Mechanistic interpretability aims to provide explanations that are faithful simplifications of the internal
computation of a trained model \citep{olah2020circuits,geiger2025causalabstraction}.
A central move in recent interpretability has been the shift from observational probes to
\emph{interventional} evaluations: we modify internal states, paths, or mechanisms and ask whether the
resulting behavior supports a proposed mechanistic explanation.

The explanations being evaluated usually fall into two broad classes. The first class consists of
\emph{abstraction or reduction fidelity claims}: a simpler object, such as a high-level causal model, compressed
network, pruned network, or circuit, is claimed to preserve the relevant causal behavior of the original model.
Causal abstraction and interchange intervention accuracy (IIA) are canonical examples
\citep{geiger2021causalabstractions,geiger2025causalabstraction}; recent work on mechanism transformations
places structured neural compression in the same interventional-fidelity setting
\citep{asiaee2026causalmechanismreductionmechanism}.
The second class consists of \emph{component-effect claims}: a set of heads, neurons, paths, or edges is
claimed to causally support a behavior. Activation patching, path patching, causal tracing, circuit ablation,
and many circuit-discovery evaluations belong to this class
\citep{meng2022rome,goldowsky_dill2023pathpatching,zhang2023patching}.
Causal scrubbing and circuit-completeness evaluations sit near the boundary: depending on the reported
score, they can be viewed either as testing a reduced explanation's fidelity or as measuring the effect of
selected components \citep{chan2022causalscrubbing}.

\paragraph{The missing piece: statistical validity.}
Despite the causal framing, evaluation practice is often informal.
A typical workflow is: sample a few thousand input pairs/prompt pairs; sample a few thousand interventions;
report a single scalar metric.
But interpretability workflows are inherently \emph{sequential} and \emph{adaptive}:
we monitor results, adjust evaluation sets, and hunt for counterexamples.
Without uncertainty quantification and without accounting for adaptivity, it is easy to overstate fidelity,
mis-rank methods, or miss rare but important failure modes.

We propose \textbf{Certified Interventional Fidelity (CIF)}, a statistical layer for turning these
interventional evaluations into certified claims. CIF is deliberately organized around the workflow an
evaluator already follows: choose what population of inputs and interventions the claim is about, run the
interventions, and report what can be concluded with uncertainty. The framework makes three design choices.
\begin{enumerate}
  \item \textbf{Make the target explicit.} Every reported score is written as an expectation over a
  declared input distribution and intervention distribution. This separates the scientific question being
  evaluated from the sampling procedure used to estimate it.
  \item \textbf{Report uncertainty that survives monitoring.} CIF uses anytime-valid confidence
  sequences, which remain valid if the evaluator checks the result repeatedly or stops when the evidence is
  strong enough \citep{howard2021confseq,waudby_smith2024confseq}.
  \item \textbf{Permit adaptive evaluation without changing the claim.} The sampler may focus over
  time on likely failures or high-impact interventions, while bounded mixture importance weighting preserves
  the original target estimand \citep{horvitzthompson1952,ville1939etude}.
\end{enumerate}

\paragraph{Positioning relative to closest prior work.}
CIF connects three lines of work: causal abstraction and interventional interpretability
\citep{geiger2025causalabstraction}, recent diagnoses of unreliable mechanistic-interpretability evaluation
\citep{meloux2025mistatistical}, and off-policy confidence sequences for adaptively sampled data
\citep{karampatziakis2021offpolicy}. The distinction is that CIF turns existing interventional
interpretability metrics into explicit causal estimands and then supplies finite-sample, anytime-valid
uncertainty guarantees for their evaluation. Extended related work is given in
Appendix~\ref{app:related}.

\paragraph{Contributions.}
The central contribution is a reformulation: interventional evaluations of mechanistic explanations can be
treated as sequential causal-inference problems. Abstraction/reduction fidelity claims and component-effect
claims ask different scientific questions, but their empirical evaluations share the same statistical form:
sample an input and an intervention, compute a bounded score, and estimate the population mean of that score.
CIF attaches uncertainty guarantees to this population claim rather than to a method-specific point estimate.
\begin{itemize}
  \item We express two broad classes of interpretability evaluations, abstraction/reduction
  fidelity and component-effect recovery, as bounded causal estimands over inputs and interventions.
  \item We give fixed-budget confidence intervals and anytime-valid confidence sequences for these
  estimands, using both transparent Hoeffding bounds and variance-adaptive betting sequences
  \citep{waudbysmith2024estimating}.
  \item We introduce adaptive CIF, which supports failure-directed sampling while preserving the
  target estimand through bounded mixture importance weighting.
  \item We provide algorithms for certification, paired comparison, adaptive sampling, and
  one-sided stopping rules, together with a reporting checklist for applying CIF to new interventional
  evaluations.
  \item We evaluate CIF on MNIST neural abstractions and GPT-2 Small IOI circuits, showing that it
  can certify high-fidelity claims, identify statistically unsupported method differences, and expose
  sensitivity to the intervention distribution.
\end{itemize}

\section{Background and notation}

\subsection{A running neural-network SCM}

We use a small feedforward network as a running example. Write its computation as
\[
H_1=f_1(X), \qquad H_2=f_2(H_1), \qquad Y=f_3(H_2),
\]
where $X$ is the random input drawn from an evaluation distribution $\cD$, $H_1$ and $H_2$ are hidden activations, and $Y$ is the output
prediction or score vector. The corresponding deterministic SCM is $\cM=(U,V,\cF)$:
$U=\{X\}$ supplies the exogenous input, $V=\{H_1,H_2,Y\}$ contains the endogenous variables, and
$\cF=\{f_1,f_2,f_3\}$ contains the structural equations
\citep{pearl2009causality,geiger2021causalabstractions,asiaee2026causalmechanismreductionmechanism}.
A realized input is denoted $x$. There is no extra randomness inside the network; for fixed $x$ and a
fixed intervention $i$, the resulting activations and output are deterministic.

In this view, an intervention changes one or more structural equations. Setting a neuron to zero
is a hard intervention on a coordinate of $H_2$; replacing the computation of a group of units by an affine
map is a soft intervention; patching an activation from a clean run into a corrupted run replaces the value of
an internal variable during that forward pass.
We write $\mathrm{State}(\cM,x,i)$ for the realized assignment of the endogenous variables after
running model $\cM$ on input $x$ under intervention $i$. It is a vector of values, not a distribution; it
becomes random only when $X$ and $I$ are sampled.

\subsection{Class I: abstraction and reduction fidelity}

In the first class, the explanation is itself a simpler causal object. The low-level model
$\cM_L$ is the original network. The high-level model $\cM_H$ is a candidate explanation produced by some
abstraction or reduction method: for example a hand-written causal model, a pruned network, an affine
compressed network, or a circuit treated as a reduced mechanism. CIF does not construct $\cM_H$; it evaluates
the candidate $\cM_H$ together with the correspondence maps supplied by the method. Two maps specify how the
objects are supposed to correspond:
\[
\tau:\mathrm{States}(\cM_L)\to \mathrm{States}(\cM_H),
\qquad
\omega:\cI_H\to \cI_L .
\]
The state map $\tau$ translates low-level activations into high-level variables. In the running
example, if $\cM_H$ keeps a subset $S$ of second-layer units, then $\tau(h_2)=h_{2,S}$. If $\cM_H$ uses a
learned affine compression, then $\tau(h_2)=Ah_2+b$. The intervention map $\omega$ translates an intervention
on the high-level variables into the corresponding intervention inside the original network. Thus $\tau$ maps
states, while $\omega$ maps interventions. The form of $\omega$ is method-dependent and is part of the
abstraction being evaluated.

The fidelity claim is interventional: $\cM_H$ should match $\cM_L$ not only on ordinary inputs, but
under corresponding manipulations of their internal variables. Schematically,
\[
\tau\!\left(\mathrm{State}(\cM_L,X,\omega(I))\right)
\approx
\mathrm{State}(\cM_H,X,I).
\]
CIF evaluates this comparison through an output-level or score-level discrepancy, which is the quantity
available in most mechanistic-interpretability experiments. The display is pointwise in realized inputs and
interventions; the distributions enter when we average the resulting scores.

\paragraph{Example 1: interchange interventions for abstraction fidelity.}
Let $X$ be a source input, $X'$ a donor input, and $M\in\{0,1\}^k$ a random mask over the high-level
coordinates of $\cM_H$. A sampled high-level intervention is $I=(M,X')$. It replaces the selected abstract
coordinates in the source run by their values from the donor run. The mapped low-level intervention
$\omega(I)$ performs the corresponding swap in $\cM_L$: for a subset abstraction it swaps the original neurons
indexed by the retained coordinates; for a subspace or affine abstraction it applies the low-level operation
specified by the abstraction map. The mask is therefore not an input mask; it indexes an internal
intervention. CIF then asks whether these matched interventions produce similar outcomes in $\cM_L$ and
$\cM_H$.

The intervention distribution $\Pi$ is the sampling rule for these tests; a typical choice is
$M_j \iid \mathrm{Bernoulli}(p)$ and $X'\sim \cD_{\mathrm{donor}}$.
Changing $p$ changes the claim being certified: $p=0.1$ asks for fidelity under mild coordinate swaps,
whereas $p=0.5$ asks for fidelity under a much more severe family of interventions.

\subsection{Class II: component-effect and recovery claims}

In the second class, the explanation is not a separate high-level model. Instead, it is a claim that
some internal component, path, or circuit causally supports a behavior of the original network. The evaluation
intervenes on $\cM_L$ itself and measures a bounded effect. Activation patching, path patching, causal tracing,
and many circuit-ablation studies have this form.

\paragraph{Example 2: activation patching for recovery.}
Let $X=(X_{\mathrm{clean}},X_{\mathrm{corr}})$ be a clean/corrupted prompt pair, and let $I=S$ denote
the component set being restored, such as a set of heads, neurons, paths, or residual-stream locations. The
patched run follows the corrupted prompt except that the activations at $S$ are copied from the clean run. If
$S$ is fixed, then $\Pi$ is a point mass; if the evaluation samples many candidate components, $\Pi$ records
how those components are chosen. As in Example 1, $\Pi$ is not a new causal semantics: it is the population of
intervention tests over which the reported metric is averaged.

\subsection{Bounded interventional scores and estimands}
\label{sec:estimand}

The two classes differ scientifically, and CIF does not force them to use the same score. What they
share is the statistical form: sample an input and an intervention, compute a bounded interventional score, and
estimate its expectation.

\paragraph{Fidelity estimands.}
For abstraction/reduction fidelity, fix a bounded discrepancy $d:\cY\times\cY\to[0,1]$ and define
\begin{multline}
Z \defeq d\!\bigl(\text{Outcome}(\cM_L, X, \omega(I)),\\
\text{Outcome}(\cM_H, X, I)\bigr)\in[0,1].
\label{eq:Zdef}
\end{multline}
Then the interventional risk is $R \defeq \E_{(X,I)\sim\cD\times\Pi}[Z]$, and fidelity is
$F \defeq 1-R$. Here small $Z$ means small disagreement, so larger $F$ means higher fidelity.

\paragraph{Recovery/effect estimands.}
For component-effect evaluations, the bounded score may already be oriented so that larger is better.
For activation patching, let $g(\cdot)$ be a scalar behavior score, such as the
indirect-object-identification (IOI) logit difference. Write $g_{\mathrm{clean}}(X)$, $g_{\mathrm{corr}}(X)$,
and $g_{\mathrm{patch}}(X,I)$ for the score on the clean run, corrupted run, and patched run. A normalized
patching recovery is
\[
\Delta(I,X)
\defeq
\mathrm{clip}_{[0,1]}\!\left(
\frac{g_{\mathrm{patch}}(X,I)-g_{\mathrm{corr}}(X)}
     {g_{\mathrm{clean}}(X)-g_{\mathrm{corr}}(X)}
\right)\in[0,1].
\]
Here $\Delta=0$ means the patch recovers none of the clean behavior, while $\Delta=1$ means full
recovery. CIF certifies the direct mean $\mu \defeq \E_{(X,I)\sim\cD\times\Pi}[\Delta]$.
Appendix~\ref{app:worked_examples} gives additional method mappings in this notation.

\section{Certified Interventional Fidelity}

Once an interventional metric has been chosen, the evaluation problem is statistical: estimate and
certify its population value from a finite number of neural-network runs, while avoiding overconfidence from
monitoring, adaptation, and selection. CIF addresses four recurring tasks: \textbf{(P1)} finite-sample
uncertainty quantification with guaranteed confidence intervals; \textbf{(P2)} sequential experimentation,
where the evaluator may inspect results and stop early; \textbf{(P3)} adaptive failure-directed sampling, where
later interventions focus on observed weaknesses; and \textbf{(P4)} selection and multiple comparisons across
many candidate components, circuits, or abstractions.

We write all targets in a common bounded-mean form. Fix an evaluation design: an input distribution
$\cD$, a target intervention distribution $\Pi$, and a bounded interventional score $Y\in[0,1]$. CIF estimates
and certifies the population mean $\theta\defeq\E_{(X,I)\sim\cD\times\Pi}[Y(X,I)]$. For abstraction fidelity,
$Y=Z$ and $\theta=R$, with fidelity $F=1-R$. For recovery or component-effect metrics, $Y=\Delta$ and
$\theta=\mu$. Thus the statistical machinery below is written for $\theta$; translating back to $F$ or $\mu$
only changes the final one-sided certification rule.

The section adds inferential machinery in the order it is needed in practice. Fixed-budget intervals
address (P1) when the sample size is fixed before looking at results. Confidence sequences address (P2) by
remaining valid under repeated monitoring and data-dependent stopping. Importance weighting addresses (P3)
without changing the target distribution $\Pi$. Betting confidence sequences keep the same anytime guarantees
but reduce certification cost when the observed variance is small. Paired inference and multiplicity control
address (P4). Ordinary fixed-$n$ central-limit-theorem intervals can be useful when $n$ is large and committed
in advance, but CIF emphasizes finite-sample, anytime-valid guarantees because interpretability evaluations are
often monitored and adapted while they run.

\subsection{Fixed-budget confidence intervals}

This is the right tool when the number of intervention tests $n$ is fixed in advance. Draw
$(X_t,I_t)\iid\cD\times\Pi$, compute $Y_t=Y(X_t,I_t)\in[0,1]$, and let
$\widehat\theta_n\defeq n^{-1}\sum_{t=1}^n Y_t$.

\begin{proposition}[Hoeffding confidence interval for a bounded interventional mean]
\label{prop:hoeffding}
For any $\delta\in(0,1)$, with probability at least $1-\delta$,
$|\widehat\theta_n-\theta|\le \sqrt{\log(2/\delta)/(2n)}$.
\end{proposition}

For fidelity, this interval is applied to $R$ and transformed to $F=1-R$; for patching recovery, it is
applied directly to $\mu$. If $Y_t$ is Bernoulli, as in 0--1 disagreement for IIA, exact binomial intervals can
also be used.

\paragraph{Compute planning.}
To achieve half-width $\varepsilon$ at confidence $1-\delta$, Hoeffding requires
$n \ge \log(2/\delta)/(2\varepsilon^2)$. This is useful when the forward-pass budget is fixed in advance, but
it does \emph{not} justify checking repeatedly and stopping when the result looks favorable.

\subsection{Anytime-valid confidence sequences (CSs)}

Interpretability evaluations are often monitored while they run: after 100, 500, or 1{,}000
interventions, we may decide whether to stop, continue, or inspect failures. Fixed-$n$ confidence intervals do
not protect this workflow. A \emph{confidence sequence} $(C_n)_{n\ge1}$ satisfies
$\Prob(\forall n\ge1:\theta\in C_n)\ge1-\delta$, so it remains valid under arbitrary stopping and repeated
monitoring \citep{howard2021confseq}.

A simple construction spends the error probability over time. Let
$\delta_n=6\delta/(\pi^2n^2)$ and $b_n=\sqrt{\log(2/\delta_n)/(2n)}$.

\begin{theorem}[Anytime Hoeffding confidence sequence via spending]
\label{thm:anytime}
With the above spending schedule, with probability at least $1-\sum_{n\ge1}\delta_n=1-\delta$,
for all $n\ge1$ simultaneously,
$\theta\in[\widehat\theta_n-b_n,\widehat\theta_n+b_n]$.
\end{theorem}

\paragraph{Certification via one-sided bounds.}
Write $U_n=\widehat\theta_n+b_n$ for the upper confidence bound (UCB) and
$L_n=\widehat\theta_n-b_n$ for the lower confidence bound (LCB). To certify fidelity $F\ge F_0$, apply the CS
to $R$ and stop when $1-U_n(R)\ge F_0$. To certify recovery $\mu\ge\mu_0$, apply the CS to $\mu$ and stop when
$L_n(\mu)\ge\mu_0$. The confidence-sequence guarantee makes either stopping rule valid even though the stopping
time is data-dependent.

\subsection{Adaptive failure-directed sampling via bounded importance sampling}
\label{sec:is}

The previous subsections assume interventions are sampled from the target distribution $\Pi$. In
practice, once failures appear, we may want to sample more interventions like them. This can find
counterexamples faster, but a naive average under the adaptive sampler estimates the wrong population. CIF uses
importance weighting to keep the estimand tied to the original $\Pi$.

\paragraph{Adaptive intervention policies.}
At time $t$, after observing the history $\cH_{t-1}$, choose a proposal distribution
$q_t(\cdot\,|\,\cH_{t-1})$ over interventions and sample $I_t\sim q_t$.
The proposal may depend arbitrarily on prior observations, enabling counterexample mining, but it must
be chosen before observing the current score $Y_t$.

\paragraph{Mixture proposals to bound weights.}
Let $\Pi$ be the \emph{target} intervention distribution and $\tilde q_t$ any adaptive auxiliary proposal.
Define the mixture $q_t \defeq (1-\alpha)\,\Pi + \alpha\, \tilde q_t$ for $\alpha\in(0,1)$.
Then the importance weight $w_t \defeq \Pi(I_t)/q_t(I_t) \le 1/(1-\alpha) \defeq W_{\max}$,
so weights remain bounded even when $\tilde q_t$ is highly concentrated.

\paragraph{Weighted estimator.}
The Horvitz--Thompson-style estimator \citep{horvitzthompson1952} is
$\widehat\theta^{\mathrm{IS}}_n\defeq n^{-1}\sum_{t=1}^n w_t Y_t$.

\begin{lemma}[Unbiasedness under adaptive proposals]
\label{lem:unbiased}
For any adaptive sequence $(q_t)$ with full support wherever $\Pi>0$,
$\E[\widehat\theta^{\mathrm{IS}}_n]=\theta$. Moreover,
$M_n\defeq\sum_{t=1}^n(w_t Y_t-\theta)$ is a martingale w.r.t.\ the history filtration.
\end{lemma}

\begin{theorem}[Anytime-valid CS under adaptive sampling]
\label{thm:adaptive}
Assume $Y_t\in[0,1]$, $w_t\le W_{\max}$ almost surely, inputs $X_t\iid\cD$, and the proposal $q_t$ is
$\cH_{t-1}$-measurable. With $\delta_n=6\delta/(\pi^2n^2)$ and
$b_n^{\mathrm{IS}}\defeq W_{\max}\sqrt{\log(2/\delta_n)/(2n)}$, we have with probability at least $1-\delta$,
for all $n\ge1$, $\theta\in[\widehat\theta^{\mathrm{IS}}_n-b_n^{\mathrm{IS}},
\widehat\theta^{\mathrm{IS}}_n+b_n^{\mathrm{IS}}]$.
\end{theorem}

\paragraph{Choosing the proposal.}
The mixture rate $\alpha$ controls an efficiency tradeoff. Larger $\alpha$ gives the auxiliary proposal
more influence, which can expose rare failures sooner, but it also increases $W_{\max}=1/(1-\alpha)$ and can
widen worst-case bounds. This choice affects efficiency, not the estimand: as long as the weights are used, the
target remains the original $\theta$ under $\Pi$. For transparency, adaptive evaluations should report
$\alpha$, describe the auxiliary proposal, and compare against an i.i.d.\ baseline when certification cost is a
central claim.

\subsection{Variance-adaptive confidence sequences via betting}
\label{sec:betting}

The Hoeffding-style CSs above are simple and transparent, but conservative: the radius depends only on
$n$ and the range of observations, not on the empirical variance. For importance-weighted observations, this
conservatism introduces a multiplicative $W_{\max}$ penalty that can make adaptive sampling \emph{slower} than
i.i.d.\ sampling. Betting CSs keep anytime validity while adapting to the observed variance.

Betting-based confidence sequences \citep{waudbysmith2024estimating} address this conservatism by adapting
to the empirical variance of the data. To avoid overloading notation, let $B_t\in[0,1]$ denote the bounded
observation passed to the betting tracker. The key idea is to construct a \emph{capital process} for each
candidate mean $m\in[0,1]$:
\begin{equation}
K_0(m) = 1, \qquad
K_n(m) = \prod_{t=1}^n \big(1 + \lambda_t(B_t - m)\big),
\label{eq:capital}
\end{equation}
where $\lambda_t$ is a ``bet'' chosen based on past data. By Ville's inequality, if
$\E[B_t\mid \cH_{t-1}]=m$ under the candidate mean, then
$\Prob(\exists\, n: K_n(m)\ge 1/\delta)\le\delta$, so the confidence set
$C_n = \{m : K_n(m) < 1/\delta\}$ is an anytime-valid CS.

We use the \emph{Online Newton Step} (ONS) betting strategy, which sets
\begin{align*}
\lambda_{n+1} &= \mathrm{clip}_{[-c,c]}\!\bigl(\lambda_n + \eta\, g_n/A_n\bigr),\\
g_n &= \tfrac{B_n - m}{1+\lambda_n(B_n-m)},\quad
A_n = A_{n-1} + g_n^2,
\end{align*}
with $\lambda_1=0$, $A_0=1$, $\eta = 2/(2-\log 3)\approx 2.22$, and $c=1/2$.
In practice, we find $C_n$ as an interval $[L_n, U_n]$ via bisection on \eqref{eq:capital},
evaluating $\log K_n(m)$ in $O(n)$ per candidate.

\paragraph{Importance-weighted extension.}
For ordinary i.i.d.\ evaluation, set $B_t=Y_t$ and the betting CS targets $\theta$. For adaptive CIF
with weights $w_t\le W_{\max}$, set
$B_t=\widetilde Y_t \defeq w_t Y_t / W_{\max} \in [0,1]$, run the betting CS on $B_t$ to obtain an interval for
$\E[B_t]=\theta/W_{\max}$, and scale back. Because betting adapts to the empirical variance of $B_t$, the
$W_{\max}$ factor no longer acts as a fixed multiplier on the radius; when $\mathrm{Var}(w_t Y_t)$ is small,
the betting CS can be much tighter than Hoeffding.

Algorithm~\ref{alg:cif} only requires a streaming tracker: after each intervention, update with the
observed score and optional importance weight, then query lower and upper confidence bounds. Hoeffding and
betting CSs therefore serve as interchangeable evaluation layers.

\subsection{Comparing two explanations with paired, anytime-valid inference}

When comparing two abstractions, circuits, or component sets, evaluate both on the same sampled
$(X_t,I_t)$. The paired difference $D_t=Y_t^{(1)}-Y_t^{(2)}\in[-1,1]$ often has much lower variance than two
separate estimates. The same CS machinery applies to $\E[D_t]$ after rescaling the range, and one can stop
early when a one-sided CS excludes $0$.

\subsection{Multiple comparisons and selection}
\label{sec:multi}

The previous guarantees are for a single pre-specified estimand. When scanning $K$ components,
circuits, or abstractions and selecting the best-looking result, uncorrected intervals are optimistic. Simple
options are to:
(i) allocate $\delta/K$ to each component (Bonferroni),
(ii) use sequential ``peeling'' (stop sampling clearly suboptimal candidates),
or (iii) use anytime-valid false discovery rate methods~\citep{wang2022fdr}.
In the experiments we recommend reporting both naive point estimates and Bonferroni-corrected CSs to show how
selection affects confidence.

\paragraph{Algorithms.}
Algorithm~\ref{alg:cif} summarizes the adaptive two-sided loop for a generic bounded score $Y$.
Algorithm~\ref{alg:certify} gives one-sided certification; the user supplies whether larger or smaller values
are favorable.

\begin{algorithm}[t]
\caption{CIF with adaptive sampling (two-sided CS)}
\label{alg:cif}
\begin{algorithmic}[1]
\Require target distribution $\Pi$; mixture rate $\alpha$; confidence $\delta$; target half-width $\varepsilon$;
adaptive proposal builder $\tilde q_t(\cdot\,|\,\cH_{t-1})$
\State $n\leftarrow 0$, $\widehat{\theta}\leftarrow 0$
\While{true}
  \State $n\leftarrow n+1$
  \State build $\tilde q_n(\cdot\,|\,\cH_{n-1})$ (e.g., AtP*/gradients or past failures)
  \State $q_n \leftarrow (1-\alpha)\Pi + \alpha \tilde q_n$
  \State sample $(X_n,I_n)\sim \cD\times q_n$, compute $Y_n\in[0,1]$
  \State $w_n \leftarrow \Pi(I_n)/q_n(I_n)$
  \State update $\widehat{\theta}\leftarrow \widehat{\theta} + \frac{1}{n}\big(w_n Y_n - \widehat{\theta}\big)$
  \State set $\delta_n\!\leftarrow\!6\delta/(\pi^2 n^2)$ and $b_n^{\mathrm{IS}} \leftarrow \frac{1}{1-\alpha}\sqrt{\frac{\log(2/\delta_n)}{2n}}$
  \If{$b_n^{\mathrm{IS}} \le \varepsilon$} \State \textbf{break} \EndIf
\EndWhile
\State \Return certified interval $[\widehat{\theta}-b_n^{\mathrm{IS}},\widehat{\theta}+b_n^{\mathrm{IS}}]$
\end{algorithmic}
\end{algorithm}

\begin{algorithm}[t]
\caption{CIF-Certify: one-sided certification}
\label{alg:certify}
\begin{algorithmic}[1]
\Require target distribution $\Pi$; confidence $\delta$; target $\theta_0$; direction larger/smaller-is-better
\State run CIF to maintain $[\widehat{\theta}_n-b_n,\widehat{\theta}_n+b_n]$
\If{larger-is-better and $\widehat{\theta}_n-b_n\ge \theta_0$} \State \Return \textbf{Certified} \EndIf
\If{smaller-is-better and $\widehat{\theta}_n+b_n\le \theta_0$} \State \Return \textbf{Certified} \EndIf
\If{budget exhausted} \State \Return \textbf{Not certified} \EndIf
\end{algorithmic}
\end{algorithm}

Proofs of the formal statements in this section are given in Appendix~\ref{app:proofs}. A short
reporting checklist is provided in Appendix~\ref{app:reporting}.

\section{Experiments}
\label{sec:experiments}

We evaluate CIF in four settings: a controlled MNIST abstraction benchmark, a GPT-2 circuit
evaluation on the IOI task, a sensitivity analysis over intervention distributions and metrics, and a coverage
check under repeated monitoring. The first two settings test whether confidence sequences can certify useful
interventional claims at realistic forward-pass budgets; the latter two check that the reported conclusions are
stable to design choices and valid under the sequential workflows CIF is meant to support.

\subsection{E1: Certified IIA for pruned and transformed abstractions}

The first experiment evaluates abstraction fidelity under interchange interventions. The goal is not to
declare one compression method universally best, but to ask which claims can be certified once uncertainty is
reported.

\paragraph{Setup.}
We use a three-layer MLP ($784 \to 512 \to 512 \to 10$) trained on MNIST to $>98\%$ test accuracy,
following the architecture used in prior structured-mechanism work
\citep{asiaee2026causalmechanismreductionmechanism}.
We construct abstractions at the second hidden layer using four methods:
(i) hard interventions via structured pruning (zeroing coordinates and removing the corresponding
columns from the downstream weight matrix),
(ii) soft interventions via affine mechanism replacement (merging neurons using learned affine maps
and compiling the result into a smaller dense network),
(iii) a variance-based pruning baseline (removing coordinates with lowest activation variance
across the training set), and
(iv) a random pruning baseline.
For each abstraction, we retain $k \in \{64, 128, 256\}$ of the 512 hidden units
(additional pruning levels are reported in Appendix~\ref{app:extra_k}),
yielding 12 reported abstraction pairs.
We evaluate under interchange interventions: for each trial we draw a pair $(x, x')$ of MNIST images
and a binary swap mask $m \in \{0,1\}^k$ (each coordinate swapped independently with probability $p=0.5$),
then compare the predictions of $\cM_L$ and $\cM_H$ under the corresponding interventions
\citep{geiger2021causalabstractions}.

\paragraph{Metrics.}
The main score is the disagreement indicator $Z \in \{0,1\}$, so $F=1-\E[Z]$ is interchange
intervention accuracy.
We also use bounded KL and $L_2$ discrepancies in the sensitivity analysis (E3).
For each method and compression level, we report:
(a) confidence sequences for $F$, and
(b) the stopping time, measured in forward passes, to certify $F \ge F_0$ for
$F_0 \in \{0.90, 0.95, 0.99\}$.

\paragraph{Adaptive proposal.}
For the adaptive runs, the auxiliary proposal places more mass on coordinates with larger downstream
weight norm, $\tilde{q}_t(j) \propto \|W_{:,j}\|_2$. We mix this proposal with the target Bernoulli-mask
distribution using $\alpha=0.3$. Appendix~\ref{app:adaptive} compares this static proposal with a
history-updated failure proposal and sweeps $\alpha$.

\paragraph{Results.}
Figure~\ref{fig:e1_cs} illustrates the main inferential pattern. Hoeffding widths depend only on
sample size and range, whereas betting intervals shrink with the empirical variance of the observed scores.

\begin{figure}[t]
\centering
\includegraphics[width=\columnwidth]{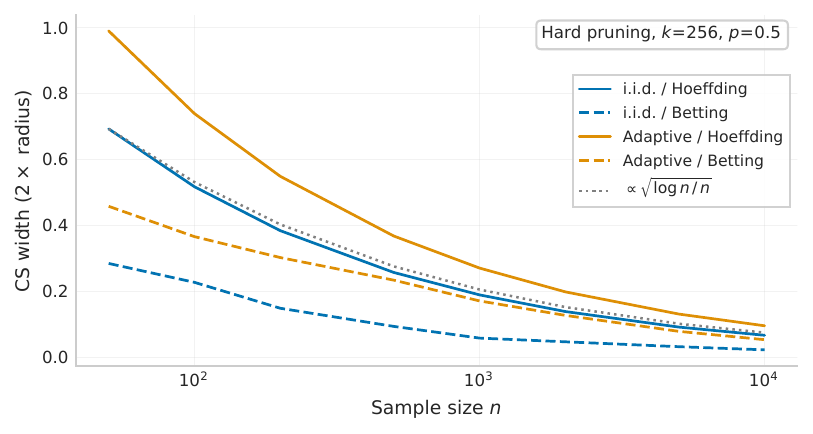}
\caption{Confidence-sequence width versus forward passes for hard pruning at $k\!=\!256$,
$p\!=\!0.5$. Solid curves use Hoeffding CSs; dashed curves use betting CSs. Blue denotes i.i.d.\ sampling,
orange denotes adaptive mixture sampling with $\alpha\!=\!0.3$, and the gray dotted curve shows a
$\sqrt{\log n / n}$ reference rate.
}
\label{fig:e1_cs}
\end{figure}

At the aggressive swap probability $p=0.5$, no non-identity abstraction reaches the certification
threshold $F\ge0.90$ (Table~\ref{tab:e1_results}). This is a useful negative result: under large coordinate
swaps, point estimates alone would still rank methods, but CIF shows that none supports a high-fidelity claim
at this threshold. A paired comparison between soft interventions and random pruning at $k=256$ further shows
why interval choice matters: the estimated difference is $-0.033$, with Hoeffding radius $0.067$ and betting
radius $0.016$.

At the milder swap probability $p=0.1$, several abstractions become certifiable
(Table~\ref{tab:e1_mild}). Variance-based pruning at $k=256$ certifies $F\ge0.90$ in 64 forward passes with
betting, compared with 1{,}141 under Hoeffding. Soft interventions at $k=256$ show the same pattern:
certification requires 93 versus 1{,}219 samples for $F\ge0.90$, and 342 versus 9{,}849 samples for
$F\ge0.95$. Figure~\ref{fig:e1_cost} summarizes these reductions at $k=256$.

Adaptive sampling is not uniformly faster for certification. In this high-fidelity regime, the
adaptive mixture roughly doubles the sample count for variance-based pruning at $k=256$ (131 samples versus
64 under i.i.d.\ betting). Appendix~\ref{app:adaptive} shows the complementary case: adaptive proposals are
more useful for finding failures than for certifying already high fidelities.

\begin{figure}[t]
\centering
\includegraphics[width=\columnwidth]{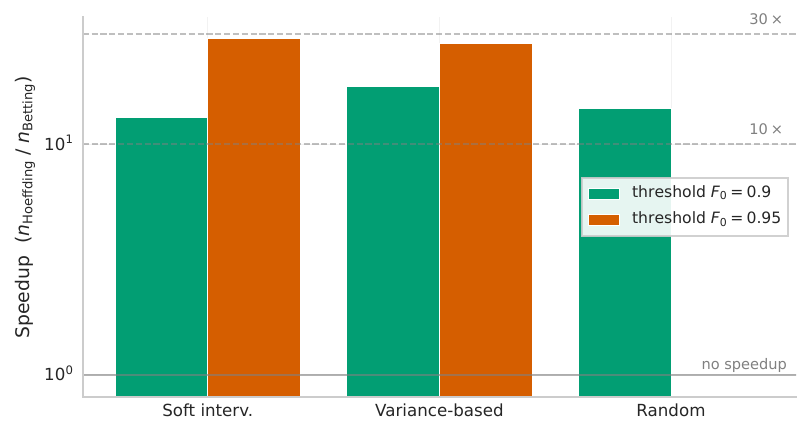}
\caption{Certification-cost ratio $n_{\text{Hoeffding}} / n_{\text{Betting}}$ for each method at
$k\!=\!256$, $p\!=\!0.1$, under i.i.d.\ sampling. Each run stops when the lower confidence bound on $F$ exceeds
the target $F_0$. Hard pruning is omitted because it never certifies; dotted reference lines mark $10\times$
and $30\times$.
}
\label{fig:e1_cost}
\end{figure}

Tables~\ref{tab:e1_results} and~\ref{tab:e1_mild} give the corresponding fidelity estimates and
stopping times.

\begin{table}[t]
\centering
\footnotesize
\setlength{\tabcolsep}{4pt}
\begin{tabular}{@{}lccc@{}}
\toprule
Method & $k\!=\!64$ & $k\!=\!128$ & $k\!=\!256$ \\
\midrule
Hard pruning   & 0.105 & 0.104 & 0.105 \\
Soft interv.   & 0.462 & 0.483 & 0.559 \\
Variance-based & 0.364 & 0.493 & \textbf{0.689} \\
Random         & 0.278 & 0.401 & 0.521 \\
\bottomrule
\end{tabular}
\caption{Estimated fidelity $\widehat{F}$ at $p\!=\!0.5$ ($n\!=\!5{,}000$,
Hoeffding CS, $1\!-\!\delta\!=\!0.95$; half-width $0.046$).
No configuration certifies $F\!\ge\!0.90$. Bold indicates the highest fidelity in the row block.
}
\label{tab:e1_results}
\end{table}

\begin{table}[t]
\centering
\footnotesize
\setlength{\tabcolsep}{2.5pt}
\begin{tabular}{@{}llcccc@{}}
\toprule
Method & $k$ & $\widehat{F}$ & \multicolumn{2}{c}{$n$ ($F\!\ge\!0.90$)} & Speedup \\
\cmidrule(lr){4-5}
 & & & Hoeff. & Bet. & \\
\midrule
Soft int.   & 64  & 0.929 & --- & 1{,}813 & $\infty$ \\
Soft int.   & 128 & 0.956 & 3{,}393 & 131 & $25.9\!\times$ \\
Soft int.   & 256 & 0.984 & 1{,}219 & 93 & $13.1\!\times$ \\
Var.-based & 64  & 0.633 & --- & --- & --- \\
Var.-based & 128 & 0.917 & --- & 4{,}725 & $\infty$ \\
Var.-based & 256 & \textbf{0.990} & 1{,}141 & \textbf{64} & $\mathbf{17.8\!\times}$ \\
Random         & 256 & 0.954 & 3{,}549 & 249 & $14.3\!\times$ \\
\midrule
\multicolumn{6}{@{}l}{\emph{$F\!\ge\!0.95$ certification:}} \\
Soft int.   & 256 & 0.984 & 9{,}849 & 342 & $28.8\!\times$ \\
Var.-based & 256 & 0.990 & 6{,}888 & \textbf{252} & $\mathbf{27.3\!\times}$ \\
\bottomrule
\end{tabular}
\caption{Certification at $p\!=\!0.1$ ($n$ up to $10{,}000$, i.i.d., $1\!-\!\delta\!=\!0.95$).
``---'' means the threshold is not certified within budget. Bold indicates the fastest certification time for
each threshold.
}
\label{tab:e1_mild}
\end{table}

\subsection{E2: Certified patching effects and circuit completeness in transformers}
\label{sec:e2}

The second experiment evaluates component-effect claims in a transformer. Here each intervention
requires a full GPT-2 Small forward pass, so certification cost is central.

\paragraph{Setup.}
We use GPT-2 Small (12 layers, 12 heads, 768-dimensional residual stream;
$\sim$124M parameters) on the Indirect Object Identification (IOI) task
\citep{wang2023ioi}.
In this task, the model must complete sentences of the form
``When Mary and John went to the store, John gave a drink to'' with the indirect object (``Mary'').
The behavior score is the IOI logit difference $g = \text{logit}(\text{IO}) - \text{logit}(\text{S})$,
where IO is the indirect object and S is the subject.
We evaluate circuits discovered via ACDC \citep{conmy2023acdc}, attribution patching
\citep{syed2024attributionpatching}, and AtP* \citep{kramar2024atpstar}, as well as the
hand-identified IOI circuit from \citet{wang2023ioi}.
Patching interventions restore a set of nodes $S$ from clean to corrupted prompts,
following best practices \citep{zhang2023patching,hanna2024faithfulness}.
We use the clipped recovery score $\Delta$ from Section~\ref{sec:estimand}; CIF certifies
$\mu=\E[\Delta]$, where larger values mean more recovery of the clean behavior.

\paragraph{Results.}
Figure~\ref{fig:e2_patching} shows that all evaluated circuits have high recovery estimates,
$\hat\mu\in[0.959,0.972]$. The distinction is inferential: at $n=2{,}000$, Hoeffding radii are about
$0.070$, while betting radii are about $0.005$. Thus the same point estimates lead to much sharper lower
confidence bounds under betting.

Table~\ref{tab:e2_cert} reports the resulting stopping times. The full 13-head circuit certifies
$\mu\ge0.90$ at $n=102$ and $\mu\ge0.95$ at $n=357$ under betting; under Hoeffding, the corresponding
requirements are $n=1{,}875$ and more than the $2{,}000$-sample budget. Even the 3-head name-mover circuit
certifies $\mu\ge0.90$ in 110 samples under betting.

\begin{figure}[t]
\centering
\includegraphics[width=\columnwidth]{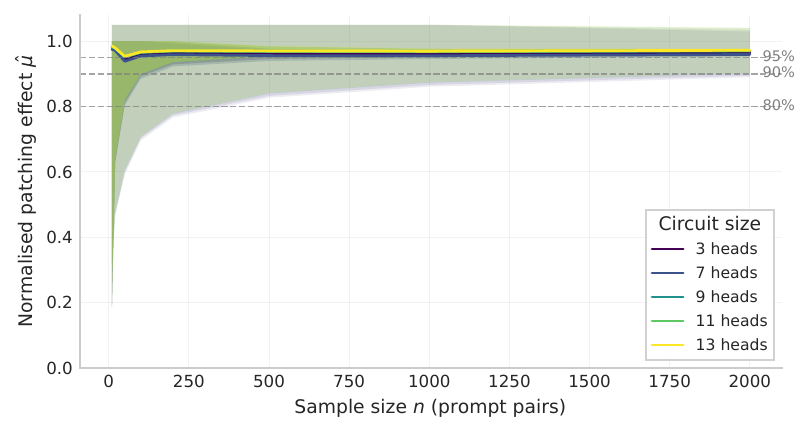}
\caption{Confidence sequences for patching recovery $\mu$ on the IOI task with GPT-2 Small. Each
curve corresponds to a circuit of increasing size (3--13 heads). Lighter bands show Hoeffding CSs; darker
bands show betting CSs ($1\!-\!\delta=0.95$).
}
\label{fig:e2_patching}
\end{figure}

\begin{table}[t]
\centering
\footnotesize
\setlength{\tabcolsep}{2.5pt}
\begin{tabular}{@{}lccccc@{}}
\toprule
Circuit & $\hat\mu$ & \multicolumn{2}{c}{$n$ ($\mu\!\ge\!0.90$)} & \multicolumn{2}{c}{$n$ ($\mu\!\ge\!0.95$)} \\
\cmidrule(lr){3-4}\cmidrule(lr){5-6}
 & & Hoeff. & Bet. & Hoeff. & Bet. \\
\midrule
3 heads  & 0.964 & --- & 110 & --- & 840 \\
7 heads  & 0.959 & --- & 118 & --- & 1{,}545 \\
9 heads  & 0.970 & 1{,}973 & 104 & --- & 413 \\
11 heads & 0.971 & 1{,}954 & 104 & --- & 369 \\
13 heads & 0.972 & 1{,}875 & \textbf{102} & --- & \textbf{357} \\
\bottomrule
\end{tabular}
\caption{Certification of IOI circuits (GPT-2 Small, i.i.d., $1\!-\!\delta\!=\!0.95$).
``---'' means the threshold is not certified within $n\!=\!2{,}000$.
}
\label{tab:e2_cert}
\end{table}

Figure~\ref{fig:e2_completeness} plots recovery against circuit size at $n=2{,}000$. Under betting,
the lower confidence bound ranges from $0.953$ for the 7-head circuit to $0.966$ for the 13-head circuit,
suggesting diminishing returns after the core name-mover heads are included.

\begin{figure}[t]
\centering
\includegraphics[width=\columnwidth]{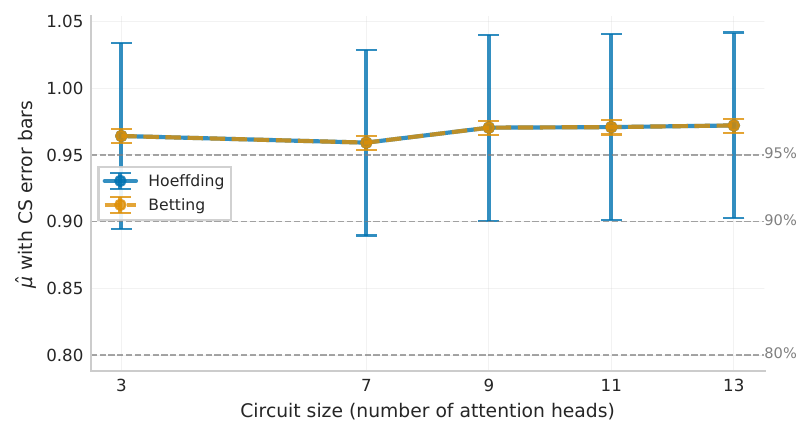}
\caption{Point estimate $\hat\mu$ and confidence sequence at $n\!=\!2{,}000$ as a function of IOI
circuit size. Solid curves show Hoeffding CSs; dashed curves show betting CSs.
}
\label{fig:e2_completeness}
\end{figure}

\subsection{E3: Sensitivity to intervention distribution and metric choice}

CIF certifies a metric under a declared intervention distribution. This experiment checks how much the
conclusion changes when that design choice changes.

\paragraph{Setup.}
We fix a model and circuit and evaluate under multiple $\Pi$:
different corruption baselines (for patching), different mask strengths (for interchange), and different
behavior metrics (probability vs logit difference vs KL) \citep{zhang2023patching}.
For each setting, we report confidence sequences and check whether point-estimate differences survive
uncertainty quantification.

\paragraph{Results.}
Figure~\ref{fig:e3_sensitivity} shows point estimates and confidence sequences across four swap probabilities
$p \in \{0.05, 0.1, 0.2, 0.5\}$ and three discrepancy metrics
(0--1 loss, clipped KL, clipped $L_2$).
At mild perturbation ($p \le 0.2$), the three metrics yield overlapping CSs:
$\widehat{F} \in [0.94, 0.99]$ with CS width $0.067$, so metric choice has no
statistically detectable effect.
At $p = 0.5$, the metrics separate: the 0--1 loss gives $\widehat{F} = 0.674$
while the clipped KL gives $\widehat{F} = 0.578$, and their CSs do \emph{not} overlap
(gap $\approx 0.03$).
At aggressive perturbation levels, the choice of discrepancy function
can therefore change qualitative conclusions; CIF makes that dependence visible by reporting intervals rather
than only point estimates.

\begin{figure}[t]
\centering
\includegraphics[width=\columnwidth]{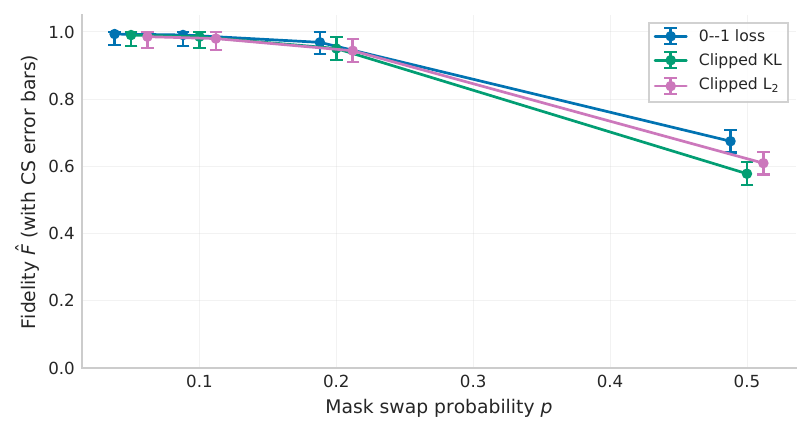}
\caption{Sensitivity of certified fidelity to swap probability $p$ and discrepancy metric $d$.
At $p\!\le\!0.2$, the confidence sequences overlap; at $p\!=\!0.5$, they separate.
}
\label{fig:e3_sensitivity}
\end{figure}

\subsection{E4: Coverage validation}

Finally, we validate coverage under i.i.d.\ sampling, adaptive sampling, and an aggressive peeking
protocol (500 runs per configuration; full details in Appendix~\ref{app:e4}). Hoeffding CSs achieve empirical
coverage $1.000$ in all configurations. Betting CSs range from $0.930$ to $1.000$, always at or above the
nominal $1-\delta$ while using less excess coverage.

\section{Discussion and limitations}

\paragraph{Choosing the target distribution $\Pi$.}
CIF makes explicit that fidelity is defined relative to a choice of intervention distribution.
We recommend reporting results across a small family of $\Pi$ (e.g., weak vs strong interventions), and using
adaptive CIF to probe likely failures without biasing headline metrics.

\paragraph{Boundedness and clipping.}
Our simplest guarantees assume bounded discrepancies/effects.
For unbounded quantities (e.g., raw logit differences), one can either clip to a bounded range or use variance-adaptive
CSs under additional assumptions \citep{howard2021confseq}.
In particular, sub-Gaussian and sub-exponential variants of time-uniform CSs
\citep{howard2021confseq}, and asymptotic CSs \citep{waudby_smith2024confseq}, extend the same
anytime-valid workflow to unbounded estimands.
Clipping is often acceptable in interpretability because extreme outliers can reflect out-of-distribution
interventions rather than meaningful causal effects.

\paragraph{What CIF does not solve.}
CIF certifies estimates for a \emph{given} metric and distribution.
It does not guarantee that the chosen metric captures the desired notion of mechanistic truth.
However, CIF makes it harder to overclaim by p-hacking sample sizes or adaptively searching without accounting
for uncertainty.
More fundamentally, CIF addresses the \emph{estimation} problem  of reliably estimating a causal
estimand from finite data \citep{imbens2015causal}, but not the \emph{identification} problem of whether the
chosen abstraction captures the causal structure of interest.

\section{Conclusion and future directions}
\label{sec:conclusion}

We have introduced Certified Interventional Fidelity (CIF), a statistical framework that provides
anytime-valid confidence sequences for interventional interpretability metrics under both fixed and
adaptive sampling regimes.
CIF treats interventional metrics (interchange intervention accuracy, patching effects, circuit
completeness scores) as formal causal estimands and wraps their evaluation with sequential
inference guarantees that remain valid under arbitrary stopping, repeated monitoring, and adaptive
counterexample hunting.
CIF has a small interface: discrepancies must be bounded (or clipped to a bounded
range), and the tracker can be applied as an evaluation layer around any existing interventional
interpretability analysis.

A key practical finding is that the choice of CS construction can change certification cost by an order of magnitude.
Hoeffding-style CSs are simple and transparent but conservative: their width depends only on $n$
and the weight bound $W_{\max}$, not on the empirical variance of the observed scores.
Betting CSs (Section~\ref{sec:betting}) adapt to the data and reduce certification cost by
$10$--$30\times$ in our experiments, from thousands of forward passes to tens or low hundreds.
This makes CIF-based certification practical even for expensive models like GPT-2, where additional
forward passes are costly.

The immediate next step is replacing the Bonferroni correction in component scanning
(Section~\ref{sec:multi}) with e-value-based FDR control \citep{wang2022fdr}, which integrates
natively with anytime-valid inference and matters most for hierarchical and per-component circuit
scans, where Bonferroni is most lossy; extending CIF to multi-layer abstractions and richer
attention-based interventions is the next challenge.
CIF is deliberately a \emph{verification} layer, the estimation side of the
estimation/identification divide \citep{imbens2015causal}, but its anytime-valid bounds are
natural controllers for \emph{discovery}: integrated into automated circuit-discovery procedures
\citep{conmy2023acdc,syed2024attributionpatching,kramar2024atpstar}, they can prune low-fidelity
candidates early and reallocate evaluation budget toward contenders, turning certification from a
post-hoc check into an active component of the search.

\begin{acknowledgements}
AA was partly supported by the Patient-Centered Outcomes Research Institute (PCORI) under
award ME-2023C1-32148 and the National Institute of Mental Health under award R01MH139379.
\end{acknowledgements}

\paragraph{Reproducibility.}
All experiments, tables, and figures in this paper can be reproduced with the code,
data-processing scripts, and notebooks available at
\url{https://github.com/AsiaeeLab/certified-interventional-fidelity}.

\bibliography{refs}

\newpage
\onecolumn
\appendix

\begin{center}
\begingroup
\hrule height4pt
\vskip .25in
{\Large\bfseries 
Certified Interventional Fidelity:\\
Anytime-Valid, Adaptive Evaluation of Causal Claims in Mechanistic Interpretability\\
(Supplementary Material)}
\vskip .25in
\hrule height1pt
\vskip .25in
{\bfseries Amir~Asiaee}\\[0.5\baselineskip]
Department of Biostatistics\\
Vanderbilt University Medical Center\\
Nashville, TN 37232, USA
\vskip .25in
\endgroup
\end{center}

\section{Extended related work}
\label{app:related}

\subsection{Causal abstraction and interventional interpretability}

Causal abstraction relates causal models at different granularities through state and intervention maps,
with exact and approximate notions of commutativity
\citep{rubenstein2017causalconsistency,beckers2019abstracting,beckers2020approximate,massidda2023soft}.
Interchange interventions operationalize these ideas in neural networks, yielding IIA as a graded faithfulness
metric \citep{geiger2021causalabstractions,geiger2022inducing,geiger2024alignments,geiger2025causalabstraction}.
\citet{wu2023interpretability} show that distributed alignment search scales causal abstraction methods to
large language models, making interchange interventions practical beyond small controlled settings.

A second strand uses causal mediation analysis inside neural networks. \citet{vig2020causalmediation}
trace gender bias through transformer layers; activation patching, path patching, and causal tracing build on
the same interventional logic
\citep{meng2022rome,goldowsky_dill2023pathpatching,zhang2023patching}. Causal scrubbing tests
interpretability hypotheses through behavior-preserving resampling interventions
\citep{chan2022causalscrubbing}.

\subsection{Circuit discovery and scalable approximations}

Automated circuit discovery methods attempt to localize important components with fewer interventions
\citep{conmy2023acdc,syed2024attributionpatching,kramar2024atpstar,hanna2024faithfulness}.
\citet{dunefsky2024transcoders} show that transcoders enable finer-grained circuit analysis than sparse
autoencoders alone, increasing the number of hypotheses that require evaluation. These methods motivate CIF
because they often scan many candidates and refine hypotheses adaptively, precisely the setting where
fixed-sample point estimates can be misleading.

Recent work also studies the reliability of circuit evaluation itself.
\citet{makelov2024interpretability} identify an ``interpretability illusion'' in subspace activation patching,
where apparently causal results arise from dormant pathways rather than the intended subspace.
\citet{miller2024faithfulness} show that circuit faithfulness metrics are sensitive to the ablation method,
complicating cross-study comparisons. \citet{shi2024hypothesis} develop formal hypothesis tests for circuit
claims, and formal verification approaches provide complementary guarantees~\citep{gross2024compact}. CIF
complements these efforts by adding sequential, anytime-valid uncertainty quantification to adaptive
interpretability workflows.

\subsection{Mechanism transformations and compression as causal operations}

Recent work views a trained network as an SCM and interprets structured pruning as a hard intervention
with an interventional-risk objective; soft interventions correspond to affine mechanism replacements and can
be compiled exactly into a smaller dense network
\citep{asiaee2026causalmechanismreductionmechanism}. CIF serves a different role: it does not propose new
abstractions, but certifies the interventional fidelity of a proposed abstraction, compression, or circuit.

\subsection{Sequential inference and confidence sequences}

\citet{ramdas2023gametheoretic} survey confidence sequences, e-values, and safe testing under the
umbrella of game-theoretic statistics. CIF draws most directly on the nonparametric confidence sequences of
\citet{howard2021confseq}, which underlie our Hoeffding baseline, and the betting-based sequences of
\citet{waudbysmith2024estimating}, whose ONS strategy we use in Section~\ref{sec:betting}. Safe testing
\citep{grunwald2024safe,shafer2021testing} and e-values~\citep{vovk2021evalues} provide related tools for
optional stopping and evidence accumulation. The e-value false-discovery-rate methods of \citet{wang2022fdr}
are a natural replacement for Bonferroni correction when scanning many components.

In the bandits literature, \citet{karampatziakis2021offpolicy} develop off-policy confidence sequences
with importance weighting, the closest statistical predecessor to adaptive CIF's mixture-importance-sampling
layer. Concurrently, \citet{meloux2025mistatistical} diagnose variance and reliability issues in mechanistic
interpretability evaluations; CIF provides one way to turn that diagnosis into certified evaluation practice.

\subsection{Importance sampling and adaptive Monte Carlo}

Importance sampling and Horvitz--Thompson estimators enable unbiased estimation under distribution
shift \citep{horvitzthompson1952}. The potential-outcomes framework for causal inference
\citep{imbens2015causal} motivates treating interventional effects as estimands: each intervention defines a
potential outcome, and the interventional risk $R$ averages these outcomes under $\Pi$. CIF combines
bounded-weight mixture proposals with martingale-based, anytime-valid bounds that remain valid when the
proposal adapts over time.

\section{Additional CIF method mappings}
\label{app:worked_examples}

Table~\ref{tab:appendix_mapping} is organized by method family, with the claim class indicating what
kind of mechanistic claim that family typically evaluates. Some circuit methods appear near the boundary
depending on whether the reported score treats the circuit as a reduced mechanism to be faithful to the full
model or as a component set whose causal effect is being measured.

\begin{table}[h]
\centering
\begingroup
\footnotesize
\renewcommand{\arraystretch}{1.15}
\setlength{\tabcolsep}{3pt}
\begin{tabular}{@{}p{0.16\textwidth}p{0.18\textwidth}p{0.24\textwidth}p{0.22\textwidth}p{0.14\textwidth}@{}}
\toprule
Claim class & Method family & Mechanistic claim & Intervention family / sampling rule & Certified quantity \\
\midrule
Abstraction / reduction fidelity
& Causal abstraction, interchange interventions, IIA
\citep{geiger2021causalabstractions,geiger2025causalabstraction}
& A high-level causal model preserves the intervention behavior of the original network.
& Source input, donor input, and mask over abstract variables; $\Pi$ samples masks and donors.
& $F=1-\E[Z]$ \\

Abstraction / reduction fidelity
& Structured pruning, affine mechanism replacement, neural compression
\citep{asiaee2026causalmechanismreductionmechanism}
& A reduced network or transformed mechanism preserves relevant causal behavior, not only ordinary predictions.
& Coordinate, mechanism, or group interventions; $\Pi$ samples masks, groups, donors, or perturbation strengths.
& $F=1-\E[Z]$ or bounded risk \\

Abstraction / reduction fidelity
& Causal scrubbing
\citep{chan2022causalscrubbing}
& A hypothesis identifies which variables are causally relevant for behavior preservation.
& Resampling interventions generated by the hypothesis; $\Pi$ samples examples and resampling choices.
& Behavior mismatch or degradation \\

Boundary case
& Circuit completeness / circuit faithfulness
\citep{wang2023ioi,hanna2024faithfulness,miller2024faithfulness,shi2024hypothesis}
& A selected circuit functions as a reduced mechanism for the behavior, or as a component set with high causal effect.
& Restore, ablate, or compare selected nodes/edges under prompt distributions.
& Fidelity $F$ or recovery $\mu$ \\

Component effect / recovery
& Activation patching, causal tracing
\citep{vig2020causalmediation,meng2022rome,zhang2023patching}
& A component causally supports a behavior because restoring it recovers the clean behavior.
& Clean/corrupted prompt pair plus component set $S$; $\Pi$ may be a point mass or a sampler over components.
& $\mu=\E[\Delta]$ \\

Component effect / recovery
& Path patching
\citep{goldowsky_dill2023pathpatching}
& A directed path or edge set carries a causal effect between upstream and downstream components.
& Clean/corrupted prompt pair plus path or edge set $E$; $\Pi$ samples paths or fixes a discovered path set.
& Recovery or clipped effect mean \\

Component effect / recovery
& Ablation studies
\citep{wang2023ioi,miller2024faithfulness}
& Removing or corrupting a component changes behavior, indicating causal relevance.
& Node, head, edge, or feature set plus ablation baseline; $\Pi$ samples components, prompts, or baselines.
& Bounded degradation or effect \\

Discovery plus evaluation
& ACDC, attribution patching, AtP*, related circuit-discovery methods
\citep{conmy2023acdc,syed2024attributionpatching,kramar2024atpstar,hanna2024faithfulness}
& A search procedure proposes candidate components or circuits; a separate evaluation estimates their causal effect or fidelity.
& Candidate-dependent proposals over nodes, edges, or paths; selection may require multiple-comparison correction.
& Same CIF estimands after selection \\
\bottomrule
\end{tabular}
\caption{Extended mapping of common interventional interpretability and reduction evaluations into CIF.
The common requirement is to declare the input distribution $\cD$, intervention distribution $\Pi$, and bounded
per-intervention score.}
\label{tab:appendix_mapping}
\endgroup
\end{table}

\paragraph{Bernoulli IIA.}
When $Z=\1\{y_L^{\omega(I)}(x)\neq y_H^I(x)\}$ indicates disagreement, $R=\Prob(Z=1)$ and
$F=1-R=\Prob(Z=0)$ is interchange intervention accuracy. Exact binomial confidence sequences can also be used;
the betting CS in Section~\ref{sec:betting} already adapts to the low variance of such observations.

\paragraph{Patching recovery.}
For activation or path patching, the patched run follows the corrupted prompt except that selected
activations are copied from the clean run. With behavior score $g$, the bounded recovery $\Delta\in[0,1]$
measures the fraction of the clean-corrupted score gap recovered by the patch, and CIF certifies
$\mu=\E[\Delta]$.

\section{Proofs}
\label{app:proofs}

We collect here the proofs of all formal results stated in the main text.

\subsection{Proof of Proposition~\ref{prop:hoeffding} (Hoeffding confidence interval for a bounded interventional mean)}

\begin{proof}
The random variables $Y_1,\ldots,Y_n$ are i.i.d.\ with $Y_t\in[0,1]$ and $\E[Y_t]=\theta$.
By Hoeffding's inequality \citep{hoeffding1963probability,howard2021confseq}, for any $\epsilon>0$,
\[
\Prob\!\left( \left|\frac{1}{n}\sum_{t=1}^n Y_t - \theta\right| > \epsilon \right)
\le 2\exp\!\left(-2n\epsilon^2\right).
\]
Setting the right-hand side equal to $\delta$ gives
$\epsilon=\sqrt{\log(2/\delta)/(2n)}$, and substituting
$\widehat\theta_n=n^{-1}\sum_{t=1}^n Y_t$ completes the proof.
\end{proof}

\subsection{Proof of Theorem~\ref{thm:anytime} (Anytime Hoeffding CS via spending)}

\begin{proof}
The proof proceeds by a union bound over the fixed-sample Hoeffding inequality
(Proposition~\ref{prop:hoeffding}) applied at each sample size $n$.

For each $n\ge1$, Proposition~\ref{prop:hoeffding} with confidence parameter $\delta_n$ gives
\[
\Prob\!\left( |\widehat\theta_n - \theta| > \sqrt{\frac{\log(2/\delta_n)}{2n}} \right) \le \delta_n.
\]
Define $b_n=\sqrt{\log(2/\delta_n)/(2n)}$. By a union bound,
\[
\Prob\!\left( \exists\, n \ge 1 : |\widehat\theta_n - \theta| > b_n \right)
\le \sum_{n=1}^{\infty} \Prob\!\left( |\widehat\theta_n - \theta| > b_n \right)
\le \sum_{n=1}^{\infty} \delta_n.
\]
The spending schedule $\delta_n=6\delta/(\pi^2n^2)$ sums to $\delta$, so
\[
\Prob\!\left( \forall\, n \ge 1 : \theta \in [\widehat\theta_n - b_n, \, \widehat\theta_n + b_n] \right)
\ge 1 - \delta.
\]
Since the event $\{\forall n\ge1:\theta\in C_n\}$ holds uniformly over time, the CS remains valid
under any possibly data-dependent stopping time $T$: if $\theta\in C_n$ for all $n$, then in particular
$\theta\in C_T$.
\end{proof}

\subsection{Proof of Lemma~\ref{lem:unbiased} (Unbiasedness under adaptive proposals)}

\begin{proof}
Fix time $t$ and condition on the history
$\cH_{t-1}=\sigma(X_1,I_1,Y_1,\ldots,X_{t-1},I_{t-1},Y_{t-1})$. Given $\cH_{t-1}$, the proposal $q_t$ is
deterministic, $X_t\sim\cD$ is independent of the history, and $I_t\sim q_t(\cdot\mid\cH_{t-1})$.

By the tower property of conditional expectation,
\begin{align}
\E[w_t Y_t \mid \cH_{t-1}]
&= \E\!\left[\frac{\Pi(I_t)}{q_t(I_t)} \cdot Y(X_t,I_t) \;\middle|\; \cH_{t-1}\right] \notag \\
&= \E_{X_t \sim \cD}\!\left[
   \E_{I_t \sim q_t}\!\left[
     \frac{\Pi(I_t)}{q_t(I_t)} \cdot Y(X_t,I_t)
     \;\middle|\; X_t, \cH_{t-1}
   \right]
\right] \notag \\
&= \E_{X_t \sim \cD}\!\left[
   \sum_{I \in \cI} q_t(I) \cdot \frac{\Pi(I)}{q_t(I)} \cdot
   Y(X_t,I)
\right] \notag \\
&= \E_{X_t \sim \cD}\!\left[
   \sum_{I \in \cI} \Pi(I) \cdot
   Y(X_t,I)
\right] \notag \\
&= \E_{X \sim \cD, \, I \sim \Pi}[Y(X,I)] = \theta.
\label{eq:tower}
\end{align}
The full-support condition ensures the ratio is well-defined, and the cancellation
$q_t(I)\Pi(I)/q_t(I)=\Pi(I)$ restores the target intervention distribution. In the continuous case, sums are
replaced by integrals.

Taking unconditional expectations gives $\E[w_t Y_t]=\theta$ for all $t$, hence
$\E[\widehat\theta_n^{\mathrm{IS}}]=\theta$.

For the martingale property, $M_0=0$ and
$\E[M_n-M_{n-1}\mid\cH_{n-1}]=\E[w_n Y_n-\theta\mid\cH_{n-1}]=0$, so $(M_n)_{n\ge0}$ is a martingale.
\end{proof}

\subsection{Proof of Theorem~\ref{thm:adaptive} (Anytime-valid CS under adaptive sampling)}

\begin{proof}
By Lemma~\ref{lem:unbiased}, $M_n=\sum_{t=1}^n(w_t Y_t-\theta)$ is a martingale.
Since $0\le w_t Y_t\le W_{\max}$ almost surely, each martingale difference has conditional range width at most
$W_{\max}$.
Hoeffding--Azuma for martingales with bounded conditional ranges gives, for any $\epsilon>0$,
\[
\Prob\!\left( |M_n| > \epsilon \right) \le
2\exp\!\left(-\frac{2\epsilon^2}{n W_{\max}^2}\right).
\]
Setting $\epsilon'= \epsilon/n$ gives
\[
\Prob\!\left( |\widehat\theta_n^{\mathrm{IS}} - \theta| > \epsilon' \right)
\le 2\exp\!\left(-\frac{2n\,\epsilon'^2}{W_{\max}^2}\right).
\]
Thus for each fixed $n$,
$\Prob(|\widehat\theta_n^{\mathrm{IS}}-\theta|>b_n^{\mathrm{IS}})\le\delta_n$, where
$b_n^{\mathrm{IS}}=W_{\max}\sqrt{\log(2/\delta_n)/(2n)}$.

The same union bound over all $n\ge1$ as in Theorem~\ref{thm:anytime} yields
\[
\Prob\!\left(\forall\, n \ge 1: \theta \in
[\widehat\theta_n^{\mathrm{IS}} - b_n^{\mathrm{IS}},\; \widehat\theta_n^{\mathrm{IS}} + b_n^{\mathrm{IS}}]\right)
\ge 1 - \sum_{n=1}^{\infty}\delta_n = 1 - \delta.
\]
\end{proof}

\section{Practical reporting checklist}
\label{app:reporting}

\begin{itemize}
  \item State $\cD$ (data/prompt distribution), $\Pi$ (intervention distribution), and the bounded score
  being averaged, such as a discrepancy $Z$ or recovery score $\Delta$.
  \item Report a confidence sequence, not just a point estimate.
  \item If you sample adaptively, report $\alpha$ and confirm that you used mixture importance sampling.
  \item When scanning many components, either correct for multiple comparisons or clearly state that reported
  intervals are \emph{post-selection} and optimistic.
\end{itemize}

\section{E1 certification at additional pruning levels}
\label{app:extra_k}

Table~\ref{tab:e1_extra_k} extends Table~\ref{tab:e1_mild} to the remaining pruning levels
$k\in\{32, 384, 512\}$ ($p=0.1$, i.i.d., $1-\delta=0.95$, $n$ up to $10{,}000$).
At $k=32$ no method reaches certifiable fidelity ($\hat F< 0.83$ everywhere).
At $k=512$ no units are removed, so all four methods coincide with the identity abstraction
and certify at identical cost.
The betting-vs-Hoeffding pattern of Table~\ref{tab:e1_mild} persists at every level where
certification is possible.

\begin{table}[h]
\centering
\begin{tabular}{@{}llccccc@{}}
\toprule
Method & $k$ & $\hat F$ & \multicolumn{2}{c}{$n$ ($F\!\ge\!0.90$)} & \multicolumn{2}{c}{$n$ ($F\!\ge\!0.95$)} \\
\cmidrule(lr){4-5}\cmidrule(lr){6-7}
 & & & Hoeff. & Bet. & Hoeff. & Bet. \\
\midrule
Soft interv.   & 32  & 0.824 & --- & --- & --- & --- \\
Soft interv.   & 384 & 0.989 & 1{,}081 & 93 & 6{,}746 & 233 \\
Hard pruning   & 32  & 0.102 & --- & --- & --- & --- \\
Hard pruning   & 384 & 0.101 & --- & --- & --- & --- \\
Variance-based & 32  & 0.494 & --- & --- & --- & --- \\
Variance-based & 384 & 0.997 & 933 & 64 & 4{,}973 & 125 \\
Random         & 32  & 0.196 & --- & --- & --- & --- \\
Random         & 384 & 0.988 & 1{,}081 & 64 & 7{,}239 & 228 \\
Identity ($k=512$, all methods) & 512 & 1.000 & 889 & 64 & 4{,}172 & 125 \\
\bottomrule
\end{tabular}
\caption{Certification at additional pruning levels ($p=0.1$, i.i.d., $1-\delta=0.95$).
``---'' = not certified within $n=10{,}000$. At $k=512$ nothing is pruned, so all methods
reduce to the identity abstraction and share one row.}
\label{tab:e1_extra_k}
\end{table}

\section{Adaptive versus i.i.d.\ sampling: matched comparison and $\alpha$-sweep}
\label{app:adaptive}

This appendix reports auxiliary adaptive-sampling studies that inform the efficiency discussion.

\paragraph{Matched three-way comparison.}
In the E1 setting, we compare three samplers under the betting CS ($\delta=0.05$, $\alpha=0.3$
for both adaptive arms):
(i) i.i.d.\ sampling from $\Pi$;
(ii) the \emph{static} proposal $\tilde q(j)\propto \|W_{:,j}\|_2$ of Section~\ref{sec:is};
and (iii) a \emph{history-updated} proposal that, after each batch of 100 samples, refits a
smoothed per-coordinate failure rate $\tilde q_t(j)\propto (s_j+1)/(n_j+2)$, where $s_j$ and
$n_j$ are the cumulative disagreement count and visit count of coordinate $j$.
Two regimes are tested on five matched seeds each:
\emph{low-fidelity exclusion} (soft interventions, $k=256$, $p=0.5$, where
$\hat F_{5000}\approx 0.55$; the task is to certify $F<0.70$)
and \emph{borderline certification} (variance-based pruning, $k=128$, $p=0.1$, where
$\hat F\approx 0.92$; the task is to certify $F\ge 0.90$).

Table~\ref{tab:adaptive_three_way} and Figure~\ref{fig:adaptive_three_way} report stopping times.
The pattern is asymmetric.
For failure-finding, the history-updated proposal (median 118, IQR 101--129) beats the static
proposal on every seed (median 201, IQR 151--219) and is competitive with i.i.d.\ (median 88,
IQR 86--150; the history proposal wins 3 of 5 matched seeds with a tighter IQR).
For certification, every adaptive proposal loses to i.i.d.\ sampling: medians 3{,}800 (history, IQR
806--4{,}031) and 1{,}708 (static) versus 1{,}217 (i.i.d.).
All samplers converge to the same $\hat F$ in both regimes, consistent with
Lemma~\ref{lem:unbiased}.
We do not have a formal result for this direction-dependence; the observed pattern matches the
one-sided importance-sampling intuition described in Section~\ref{sec:is}.

\begin{table}[h]
\centering
\begin{tabular}{@{}lccc@{}}
\toprule
 & i.i.d. & Static adaptive & History adaptive \\
\midrule
Exclusion $n$ (median, IQR) & \textbf{88} (86--150) & 201 (151--219) & 118 (101--129) \\
Certification $n$ (median, IQR) & \textbf{1{,}217} (763--3{,}263) & 1{,}708 (1{,}501--3{,}007) & 3{,}800 (806--4{,}031) \\
\bottomrule
\end{tabular}
\caption{Matched five-seed comparison of samplers in the two E1 regimes (betting CS,
$\alpha=0.3$). Low-fidelity exclusion: soft interventions, $k=256$, $p=0.5$, certify $F<0.70$.
Borderline certification: variance-based pruning, $k=128$, $p=0.1$, certify $F\ge 0.90$.}
\label{tab:adaptive_three_way}
\end{table}

\begin{figure}[h]
\centering
\includegraphics[width=0.85\textwidth]{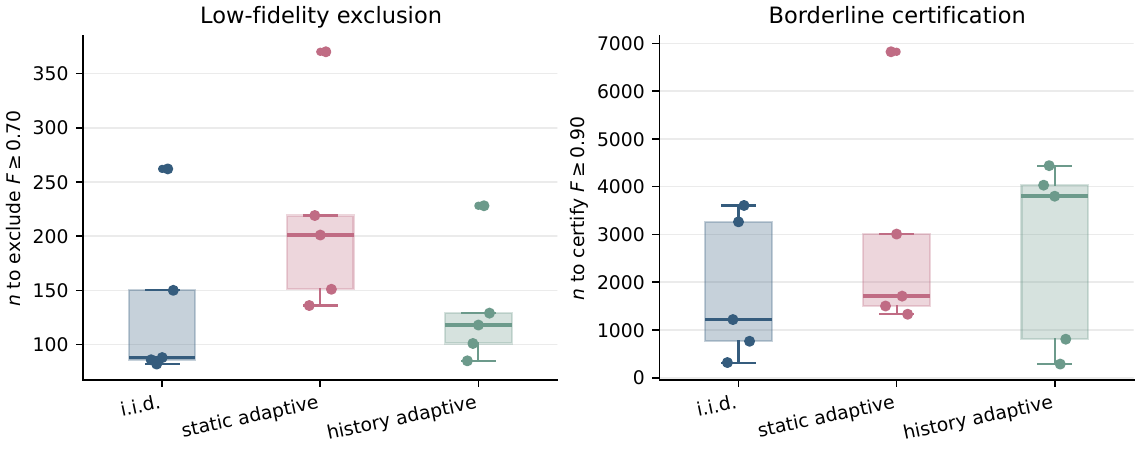}
\caption{Stopping times across five matched seeds per sampler (boxes: IQR; dots: individual
seeds). Left: low-fidelity exclusion ($n$ to certify $F<0.70$). Right: borderline certification
($n$ to certify $F\ge 0.90$; note the static-adaptive outlier at $n=6{,}824$).
History-updated adaptive sampling helps failure-finding (left) and hurts certification (right).}
\label{fig:adaptive_three_way}
\end{figure}

\paragraph{Sweep over the mixture rate $\alpha$.}
On the borderline-certification configuration of Table~\ref{tab:e1_mild} (variance-based pruning,
$k=256$, $p=0.1$, betting CS, three seeds per point), we sweep
$\alpha\in\{0, 0.1, 0.2, 0.3, 0.5, 0.7\}$ with the static proposal.
Figure~\ref{fig:alpha_sweep} shows median certification cost with IQR bars.
For $F_0=0.90$ the medians are 64, 96, 101, 121, 156, and 255 samples;
for $F_0=0.95$ they are 223, 190, 199, 240, 328, and 575.
The non-monotonicity at small $\alpha$ is within three-seed noise; the slowdown above
$\alpha=0.3$ is systematic and matches the $W_{\max}=1/(1-\alpha)$ weight-inflation bound
(at $\alpha=0.7$, $W_{\max}=3.33$ and certification at $F_0=0.95$ is $2.6\times$ slower than
i.i.d.).

\begin{figure}[h]
\centering
\includegraphics[width=0.55\textwidth]{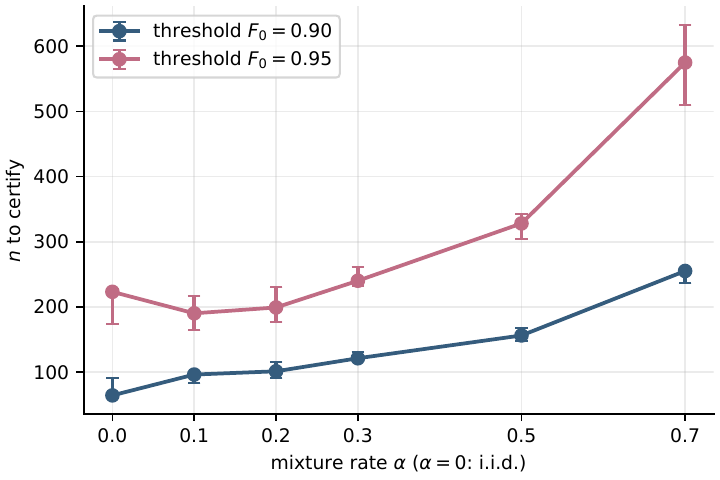}
\caption{Median certification cost versus mixture rate $\alpha$ (variance-based pruning,
$k=256$, $p=0.1$, betting CS; three seeds per point; bars: IQR). $F_0$ is the certification
threshold: each run terminates when the lower confidence bound on $F$ exceeds $F_0$.}
\label{fig:alpha_sweep}
\end{figure}

\section{E4: Coverage validation details}
\label{app:e4}

\paragraph{Setup.}
In the MNIST MLP setting (E1), we fix one abstraction and estimate $R$ across 500 independent runs
under three protocols: i.i.d.\ sampling from $\Pi$, adaptive mixture sampling with $\alpha=0.3$, and aggressive
peeking, where the interval is checked after every 10 samples and the run stops at the first exclusion of a
pre-specified value. For each protocol, we record whether the true $R$ (estimated by Monte Carlo with
$10^5$ i.i.d.\ samples) falls inside the confidence sequence at the stopping time.

\paragraph{Results.}
Table~\ref{tab:e4_coverage} reports empirical coverage for Hoeffding and betting CSs across three
confidence levels. Hoeffding coverage is $1.000$ in every configuration, reflecting the conservatism of the
range-based bound. Betting coverage ranges from $0.930$ to $1.000$, always at or above the nominal guarantee
while leaving less unused error budget.

\begin{table}[h]
\centering
\small
\begin{tabular}{@{}llcccc@{}}
\toprule
Sampling regime & $1-\delta$ & \multicolumn{2}{c}{Hoeffding} & \multicolumn{2}{c}{Betting} \\
\cmidrule(lr){3-4}\cmidrule(lr){5-6}
 & & Coverage & Std & Coverage & Std \\
\midrule
i.i.d.\ ($\alpha=0$) & 0.99 & 1.000 & 0.000 & 1.000 & 0.000 \\
Adaptive ($\alpha=0.3$) & 0.99 & 1.000 & 0.000 & 1.000 & 0.000 \\
Aggressive peeking & 0.99 & 1.000 & 0.000 & 0.992 & 0.004 \\
\midrule
i.i.d.\ ($\alpha=0$) & 0.95 & 1.000 & 0.000 & 0.998 & 0.002 \\
Adaptive ($\alpha=0.3$) & 0.95 & 1.000 & 0.000 & 1.000 & 0.000 \\
Aggressive peeking & 0.95 & 1.000 & 0.000 & 0.966 & 0.008 \\
\midrule
i.i.d.\ ($\alpha=0$) & 0.90 & 1.000 & 0.000 & 0.996 & 0.003 \\
Adaptive ($\alpha=0.3$) & 0.90 & 1.000 & 0.000 & 0.996 & 0.003 \\
Aggressive peeking & 0.90 & 1.000 & 0.000 & 0.930 & 0.011 \\
\bottomrule
\end{tabular}
\caption{Empirical coverage over 500 runs with $n\!=\!2{,}000$ samples per run. Hoeffding is
conservative in all configurations; betting remains at or above nominal coverage while producing tighter
intervals.}
\label{tab:e4_coverage}
\end{table}

\end{document}